\documentclass[10pt,twocolumn,letterpaper]{article}

\usepackage{cvpr}              
\usepackage[accsupp]{axessibility}
\usepackage[table,dvipsnames]{xcolor}
\usepackage{makecell}
\usepackage{tcolorbox}
\usepackage{enumerate}

\usepackage[shortlabels]{enumitem}
\usepackage{amsmath}
\usepackage{diagbox}
\usepackage{amssymb}
\usepackage{booktabs}
\usepackage{multirow}
\usepackage{times}
\usepackage{lipsum}
\usepackage{subcaption}
\usepackage{diagbox}
\usepackage{algorithm}
\usepackage{algpseudocode}

\usepackage{bm,cuted}
\usepackage{graphicx}
\usepackage{tabularx} 
\usepackage[safe]{tipa}
\usepackage{multirow}
\usepackage{xcolor}
\usepackage{dsfont}
\usepackage{upgreek}
\usepackage{url}

\usepackage{pifont}
\newcommand{\cmark}{\ding{51}}%
\newcommand{\xmark}{\ding{55}}%
\makeatletter
\newcommand\HUGE{\@setfontsize\Huge{20}{30}}
\makeatother

\def\eqref#1{equation~\ref{#1}}

\def\1{\bm{1}}

\DeclareMathAlphabet{\mathsfit}{\encodingdefault}{\sfdefault}{m}{sl}
\SetMathAlphabet{\mathsfit}{bold}{\encodingdefault}{\sfdefault}{bx}{n}

\usepackage[accsupp]{axessibility}
\usepackage{amsthm}

\providecommand{\FullStop}{\text{~\@.\xspace}}
\providecommand{\Comma}{\text{~,\xspace}}

\newtheorem{theorem}{Theorem}

\theoremstyle{remark}
\newtheorem{remark}{Remark}[theorem]

\usepackage{xspace}
\usepackage{hhline}   
\usepackage{siunitx}  

\usepackage{graphicx}
\graphicspath{{./images/}}

\makeatletter
\let\NAT@parse\undefined
\makeatother
\usepackage[bookmarks=false,linkcolor=blue, urlcolor=blue, citecolor=blue]{hyperref} 

\hypersetup{
    colorlinks=true,
    linkcolor=red,
    filecolor=magenta,      
    urlcolor=blue,
    pdfstartview={FitH},
    citecolor =blue
    }

\usepackage[capitalize]{cleveref}
\crefname{section}{Sec.}{Secs.}
\Crefname{section}{Section}{Sections}
\Crefname{table}{Table}{Tables}
\crefname{table}{Tab.}{Tabs.}

\def\cvprPaperID{7631} 
\def\confName{ICCV}
\def\confYear{2025}

\begin{document}










\title{EgoMusic-driven Human Dance Motion Estimation with Skeleton Mamba}
\author{Quang Nguyen$^{1}$, Nhat Le$^{2}$, Baoru Huang$^{7,*}$, Minh Nhat Vu$^{3}$, Chengcheng Tang$^{4}$, \\ Van Nguyen$^{1}$, Ngan Le$^{5}$, Thieu Vo$^{6}$, Anh Nguyen$^{7}$\\
{\small $^{1}$FPT Software AI Center}
{\small $^{2}$The University of Western Australia}
{\small $^{3}$TU Wien} 
{\small $^{4}$Meta} \\
{\small $^{5}$University of Arkansas}
{\small $^{6}$National University of Singapore}
{\small $^{7}$University of Liverpool}
{\small $^{*}$Corresponding author}\\
{\small \href{https://zquang2202.github.io/SkeletonMamba/}{https://zquang2202.github.io/SkeletonMamba/}}}


\twocolumn[{%
\renewcommand\twocolumn[1][]{#1}%
\maketitle
\begin{center}
  \centering
  \vspace{-5.5ex}
  \captionsetup{type=figure}
  \Large
\resizebox{\linewidth}{!}{
\setlength{\tabcolsep}{2pt}
\begin{tabular}{ccccc}

\shortstack{\includegraphics[width=0.9\linewidth]{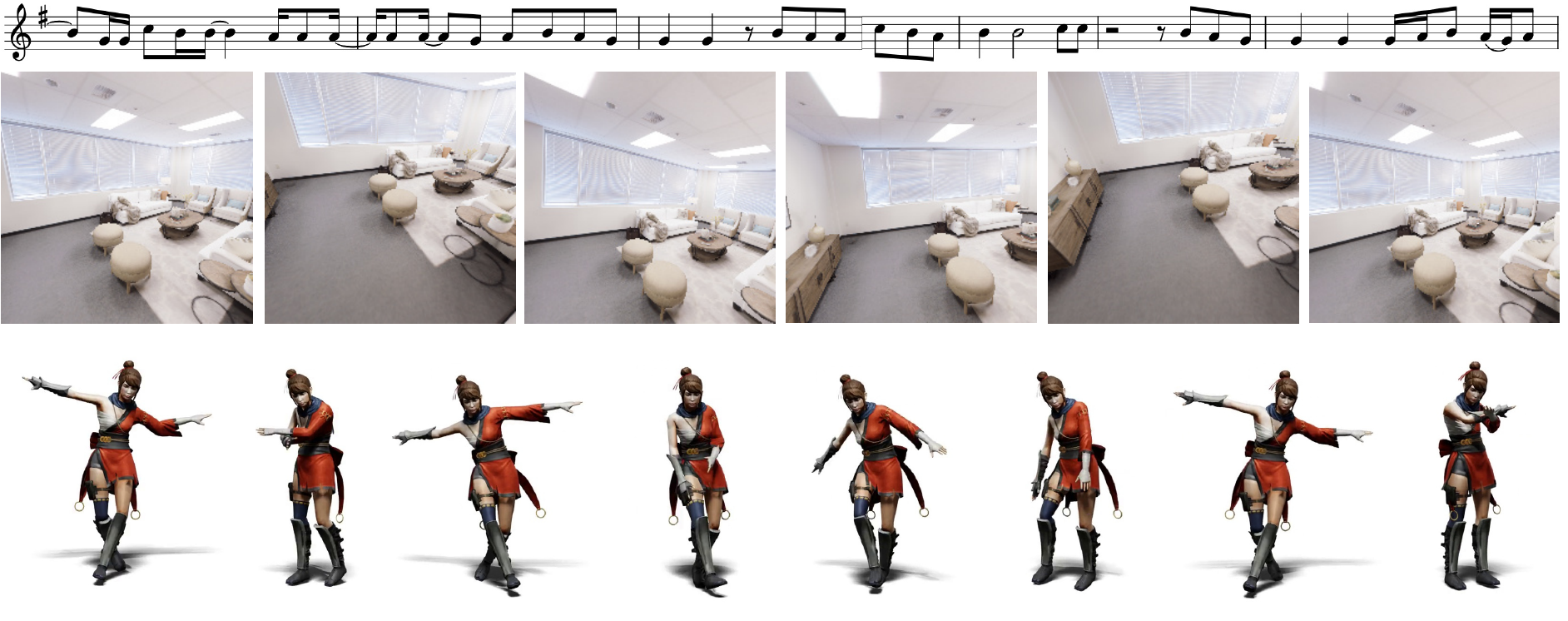}}\\[2.5pt]

\end{tabular}
}
\vspace{-4ex}
    \captionof{figure}{We present a new dataset and method for estimating human dance motion from the egocentric video and music.}
    \label{fig: IntroVis}
\end{center}%
    }]


\begin{abstract}
\vspace{-5ex}
Estimating human dance motion is a challenging task with various industrial applications. Recently, many efforts have focused on predicting human dance motion using either egocentric video or music as input. However, the task of jointly estimating human motion from both egocentric video and music remains largely unexplored. In this paper, we aim to develop a new method that predicts human dance motion from both egocentric video and music. In practice, the egocentric view often obscures much of the body, making accurate full-pose estimation challenging. Additionally, incorporating music requires the generated head and body movements to align well with both visual and musical inputs. We first introduce EgoAIST++, a new large-scale dataset that combines both egocentric views and music with more than 36 hours of dancing motion. Drawing on the success of diffusion models and Mamba on modeling sequences, we develop an EgoMusic Motion Network with a core Skeleton Mamba that explicitly captures the skeleton structure of the human body. We illustrate that our approach is theoretically supportive. Intensive experiments show that our method clearly outperforms state-of-the-art approaches and generalizes effectively to real-world data. 
\end{abstract}
\vspace{-4ex}

\section{Introduction}
\vspace{-2ex}

Dance is a fundamental form of human expression, creativity, and is deeply embedded in cultural and social contexts~\cite{hanna1987dance, lamothe2019dancing}. Estimating full-body dance motion is a crucial task with many industrial applications, such as dance education~\cite{guo2022dancevis, difini2021human, le2023controllable}, virtual metaverses~\cite{lam2022human, kim2023performing}, or film animation~\cite{chan2019everybody, wang2024disco}. While several works have focused on human dance pose estimation, they mostly tackle the problem using the input from third-person video ~\cite{hu2021unsupervised, bera2023fine, rani2022effectual, zhao2019dance, qianwen2024application} or motion tracking device~\cite{matsuyama2019hybrid, matsuyama2021deep}. In practice, third-person video methods suffer from occlusions, viewpoint variations, and depth ambiguity, while motion-tracking devices require costly hardware, making them less practical for real-world AR/VR or metaverse applications. These constraints highlight the need for an alternative approach, such as first-person (egocentric) view dance motion estimation.


Recently, egocentric input has been utilized to estimate human motions in everyday activities such as walking and running~\cite{jiang2021egocentric, wang2021estimating, wang2024egocentric, millerdurai2024eventego3d}. 
A key challenge in egocentric pose estimation is that much of the body often falls outside the camera’s view, creating ambiguities in full-body capture. This issue is even more pronounced in dance motion, where movements are more complex and dynamic. To overcome this, many works incorporate motion tracking sensors, which require costly hardware and limit accessibility~\cite{matsuyama2019hybrid, matsuyama2021deep}. 
A promising approach for accurate dance motion estimation is to leverage music as an additional modality. Research has shown that music, through its rhythm, tempo, and dynamics, can provide meaningful cues for generating realistic movements~\cite{lamothe2019dancing, brown2008neuroscience}. 
Therefore, we propose integrating \textit{music and egocentric video} for dance motion estimation. We hypothesize that combining these two complementary modalities enables more accurate human dance motion estimation. However, the challenge lies in two fundamental tasks: \textit{i)} creating a large-scale dataset for dance motion estimation from egocentric and music, and \textit{ii)} designing a model capable of understanding the human skeleton structure and effectively coordinating multimodal inputs to synchronize head and body motion with the music


To learn human motion, transformer is a widely used technique~\cite{zhang2022motiondiffuse, chen2023executing, tevet2023human, tseng2023edge, li2021ai}. However, transformer has quadratic complexity and struggles to capture structure dependencies. On the other hand,  State Space Models (or Mamba)~\cite{mamba2, gu2023mamba} have shown great potential on several tasks over transformer-based models, such as graph analysis~\cite{wang2024graph}, video analysis~\cite{li2025videomamba}, and image generation~\cite{hu2024zigma}. However, directly adapting Mamba to human motion data presents a challenge due to the dynamics of both spatial and temporal structures of the human body.
Previous Mamba-based models~\cite{zhang2024motion, wang2024text, qian2024smcd, zhang2024infinimotion} on human motion usually simplify each frame of a human pose as a single latent vector or disregard the spatial order of joints within the human skeleton, which limits their ability to capture fine-grained spatial dynamics. This drawback makes it challenging to generate coherent motion as the head and lower body may fail to align naturally, leading to poor coordination between egocentric-driven head and music-influenced body movement. 

In this work, we first introduce a new dataset for human dance pose estimation from egocentric and music inputs. We then propose Skeleton Mamba, a new Mamba model designed to capture spatial structures while preserving temporal coherence. Our designed method enables synchronized head and body movements responsive to egocentric and music inputs.  Our approach explicitly models the spatial structure of joints and their hierarchical dependencies, allowing for more coherent motion generation that preserves the natural relationships between joints. We show that our method is theoretically supportive, and provide intensive experiments to validate our method against recent state-of-the-art approaches. Our contributions are the following:
\begin{itemize}[leftmargin=*]
    \item We propose a new dataset for human dance motion estimation from the egocentric and music input. 
    \item We propose the EgoMusic Motion Network with Skeleton Mamba as the core to learn human body motion.
    \item We provide theoretical analysis and intensive experiments to demonstrate the effectiveness of our method.
\end{itemize}

\section{Related Work}

\paragraph{Human Motion from Egocentric Video.} Human motion estimation from egocentric video has garnered significant attention in recent years. Most existing methods assume partial visibility of body parts in the image, often using fish-eye cameras~\cite{jiang2021egocentric, wang2021estimating, wang2024egocentric, millerdurai2024eventego3d, tome2019xr, xu2019mo2cap2}. Other research addresses the challenge of body parts not being visible in egocentric footage~\cite{jiang2017seeing, luo2021dynamics, ng2020you2me, li2023ego}. Jiang \textit{et al.}~\cite{jiang2017seeing} introduce an innovative global optimization technique that utilizes both trained dynamic and scene classifiers along with pose coupling over an extended period. Ng \textit{et al.}~\cite{ng2020you2me} model person-to-person interactions, inferring the 3D ego-pose based on the other person’s pose. Luo \textit{et al.}~\cite{luo2021dynamics} jointly models kinematics and dynamics to estimate 3D human poses and human-object interactions. EgoFormer~\cite{li2023egoformer} extracts motion features from egocentric images and employs a Transformer Decoder to autoregressively generate human poses. In~\cite{li2023ego}, the authors introduce  EgoEgo, a hybrid learning method for head pose estimation, which then was used as a conditioning factor in a diffusion model to estimate full-body motions. 


\vspace{-0.5cm}
\paragraph{Human Motion from Music.} 
Generating human dance motion from Music is widely formed as a synthesis task. Early studies used statistical retrieval techniques to generate choreography by seamlessly transitioning between existing motion clips~\cite{fan2011example, lee2013music}. However, these methods rely on selecting pre-existing motions, often resulting in unnatural dance motions. With advancements in deep learning techniques and the availability of large-scale datasets, many networks have been introduced to generate higher-fidelity dance motions~\cite{au2022choreograph, siyao2022bailando, li2021ai, huang2022genre, li2022danceformer, fan2022bi, le2023music,le2024scalable}. The FACT model~\cite{li2021ai} introduces an autoregressive cross-modal transformer to generate long continuous dance sequences. Bailando~\cite{siyao2022bailando} employs VQ-VAEs for the upper and lower body segments to translate music and initial poses into dance sequences. Recent efforts have explored the use of diffusion models for dance generation~\cite{li2021ai, zhang2024bidirectional}. MoFusion~\cite{dabral2023mofusion} presents a multi-condition diffusion framework capable of generating long, realistic, and temporally coherent human motion sequences. EDGE~\cite{tseng2023edge} introduces an editable dance generation model that leverages a transformer-based diffusion architecture, offering flexible editing capabilities for dance applications. In~\cite{beatit}, the authors propose a framework that allows control of generated dance motion based on key-frame body pose and music beat conditions. However, all these works focus on generating dance motion using only the music as input, and the egocentric views are not taken into account. In this work, we present a new diffusion framework that aligns body movements with both music and egocentric video. 

\vspace{-0.5cm}
\paragraph{State Space Model.} 
State Space Model (SSM)~\cite{gu2023mamba} has garnered significant attention recently due to its potential for efficiently modeling long sequences with linear complexity. Its applications have been explored across various fields, including image processing~\cite{hu2024zigma, wang2024mamba, zhu2024vision, li2024mamba,lee2025meteor}, graph processing~\cite{wang2024graph, behrouz2024graph}, point cloud analysis~\cite{liang2024pointmamba, liu2024point, zhang2025voxel}, and human motion generation~\cite{zhang2024motion}. Vim~\cite{zhu2024vision} presents a bidirectional SSM block. Efforts like Mamba-ND~\cite{li2024mamba} extend the capabilities of SSM to higher-dimensional data by exploring different scan directions within a single SSM block. In addition, several works, including ZigMa~\cite{hu2024zigma} and DiffuSSM~\cite{yan2024diffusion}, utilize Mamba-based SSM blocks for efficient image generation. MotionMamba~\cite{zhang2024motion} proposes a symmetric multi-branch Mamba that processes temporal and spatial and shows exceptional performance on text-to-motion generation tasks. More recently, State Space Duality (SSD)~\cite{mamba2} is introduced as a dual-form framework that unifies state space models with structured masked attention. 

\section{The EgoAIST++ Dataset}
\label{sec_dataset}
While several datasets have been proposed for single input (either egocentric or music) human motion estimation (Table~\ref{tab: dataset_compare}), large-scale datasets that combine egocentric and music for dance pose are still limited. Ego-Exo4D dataset~\cite{grauman2024ego} has a subset with the egocentric view, music, and dance motion, but this subset is only approximately 2 hours. To address this gap, we introduce EgoAIST++, a new large-scale dataset that integrates egocentric views and music specifically for human dance pose estimation.

\begin{table}[h]
\centering
\small
\setlength{\tabcolsep}{0.15 em} 
\resizebox{\linewidth}{!}{
\begin{tabular}{lcccccc}
\toprule
Datasets & Music & Egocentric & Setup & \#Images & Camera Direction \\ \midrule
$\text{Mo}^2\text{Cap}^2$~\cite{xu2019mo2cap2} & \xmark & \cmark & Mocap & 530k & Downward-facing \\
xr-EgoPose~\cite{tome2019xr} & \xmark & \cmark & Simulation & 380k & Downward-facing \\
UnrealEgo~\cite{hakada2022unrealego} & \xmark & \cmark & Simulation & 450k & Downward-facing \\
EgoGTA~\cite{egogta} & \xmark & \cmark & Simulation & 320K & Downward-facing \\
EgoBody3M~\cite{egobody3M} & \xmark & \cmark & Mocap & 3.4M & Downward-facing \\
ECHP~\cite{echp} & \xmark & \cmark & Mocap & 75k & Downward-facing \\
ARES~\cite{li2023ego} & \xmark & \cmark & Simulation & 1.6M & Forward-facing \\
Ego-Exo4D~\cite{grauman2024ego} & \cmark & \cmark & Mocap & / & Forward-facing \\ \midrule
DanceNet~\cite{dancenet} & \cmark & \xmark & - & - & - \\
EA-MUD~\cite{EAMUD} & \cmark & \xmark & - & - & - \\
AIST++~\cite{li2021ai} & \cmark & \xmark & - & - & - \\ \midrule
EgoAIST++ (ours) & \cmark & \cmark & Mixed & 3.9M & Forward-facing \\ \midrule
\end{tabular}}
\vspace{-2ex}
\caption{\textbf{Human motion datasets comparison.}}
\label{tab: dataset_compare}
\vspace{-1ex}
\end{table}

\textbf{Setup.} 
We first utilize the AIST++ dataset~\cite{li2021ai} as it includes real-world well-defined human motion paired with the music. The data from AIST++ dataset was captured using a motion capture system which ensures the correctness of the human motion. We then use the Replica 3D indoor scene dataset~\cite{straub2019replica} to provide environment for obtaining the visual egocentric view. We randomly placed AIST++~\cite{li2021ai} sequences with a specified location and rotation in the 3D mesh scene of the Replica dataset.  For each sequence, we enforce the penetration constraint as in~\cite{wang2021synthesizing} to maintain the natural contact between the human and the objects in the scene. Other factors such as collision with surrounding objects are resolved manually by human annotators.



\textbf{Data Labelling and Statistic.}  We use AI Habitat~\cite{szot2021habitat}, to render high-quality and realistic egocentric images from a head-mounted camera of a virtual human and 3D mesh scene. We split the data from the AIST++ and Replica datasets to ensure the train and test sets have distinct music choreographies and scenes, with no overlap between them. The test set includes 40 unique music choreographies and 5 distinct scenes, while the training set consists of 980 music choreographies and 13 scenes for training. For each dance sequence, we divide it into 5-second subsequences and place them at a random location within the scene. Overall, our EgoAIST++ dataset has $36$ hours of motion with nearly 3.9M frames, recorded at 30 frames per second.

\section{EgoMusic-driven Dance Motion Estimation}
\begin{figure*}[ht]
    \centering
    \includegraphics[width=\linewidth]{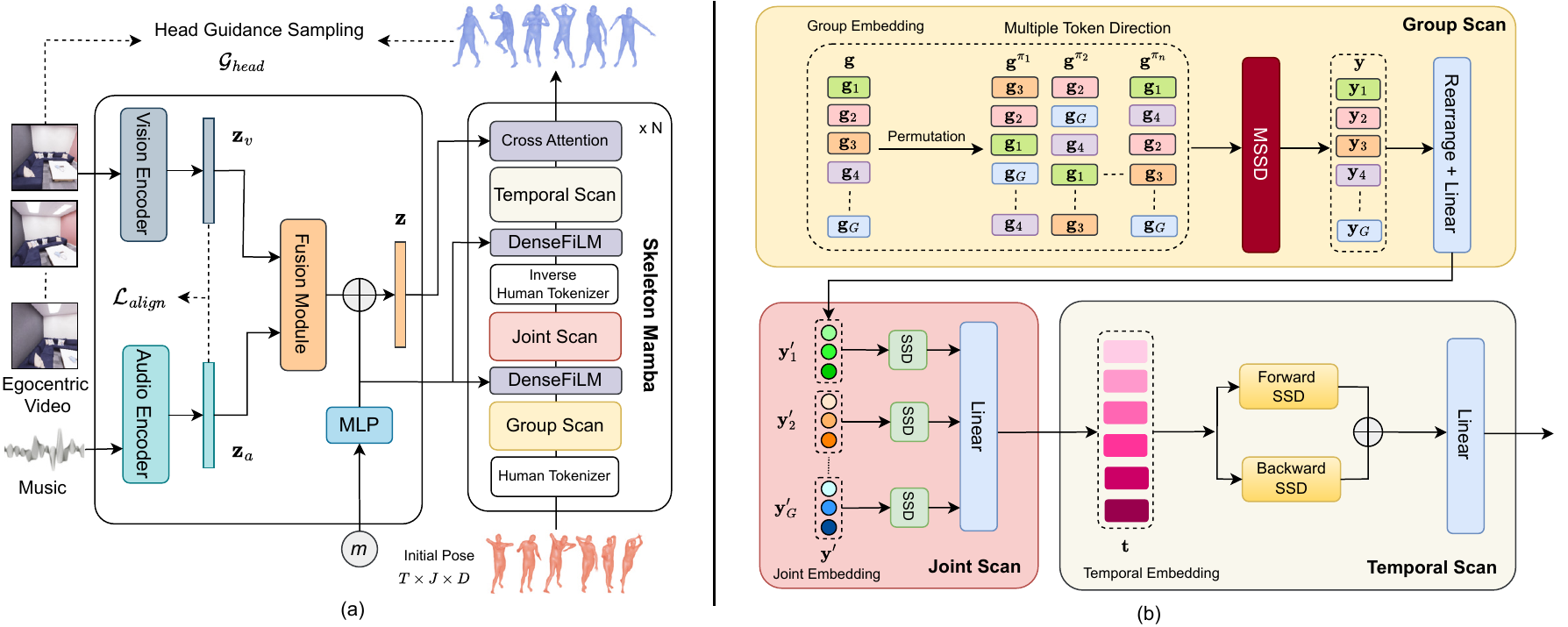}
    \vspace{-2ex}
    \caption{\textbf{Methodology overview.} (a) We propose a new diffusion model framework that generates human motion from egocentric video and music. (b) Detail architecture of three main components: Group Scan, Joint Scan, and Temporal Scan. Our model can effectively capture both the spatial and temporal dynamics in human motion data.}
    \label{fig: model}
    \vspace{-2ex}
\end{figure*}

\subsection{Problem Formulation}
Given an egocentric video represented as a sequence of frames, $\mathbf{v} = \{\mathbf{v}^1, \mathbf{v}^2, \dots, \mathbf{v}^T\}$, and a piece of music, $\mathbf{a} = \{\mathbf{a}^1, \mathbf{a}^2, \dots, \mathbf{a}^T\}$, both of duration $T$, our objective is to generate a human dance sequence, $\mathbf{x} = \{\mathbf{x}^1, \mathbf{x}^2, \dots, \mathbf{x}^T\}$, that aligns with the egocentric view and the audio. As in~\cite{li2023ego, tseng2023edge}, the human pose $\mathbf{x}$ is presented using the SMPL model~\cite{loper2023smpl}. We formulate our problem using a condition diffusion model \cite{ho2020denoising, dhariwal2021diffusion} where we represent the target motion as $\mathbf{x}_0$, and combine egocentric images $\mathbf{v}$ and music $\mathbf{a}$ as the condition $\mathbf{z}$ in the diffusion model. The objective of the diffusion process~\cite{ho2020denoising} is to gradually add noise into a clean dance motion \(\mathbf{x}_0 \) over a series of \( m \) steps:
\begin{equation}\label{eq: forward}
 q(\mathbf{x}_m | \mathbf{x}_0) = \mathcal{N}( \sqrt{\alpha_m} \mathbf{x}_0, (1 - \alpha_m) \mathbf{I} )\Comma
\end{equation}
where \( \alpha_m = \prod_{s=1}^{m} (1 - \beta_s) \), and \( \beta_m \) controls the noise schedule. The backward process is to learn the condition distribution \( p_\theta(\mathbf{x}_0 | \mathbf{z})\) with the condition $\mathbf{z}$ as in~\cite{ho2020denoising, dhariwal2021diffusion}. 



Our objective is to design an effective backward process. To generate human motion that is well-aligned with egocentric cues and music, we introduce a new EgoMusic Motion Network (EMM) centered around a core Skeleton Mamba scanning strategy to learn the structure of the human body. We provide empirical and theoretical evidence demonstrating the contribution of our Skeleton Mamba in learning human motion during the denoising process.  


\subsection{EgoMusic Motion Network}
The overall pipeline of our proposed method is illustrated in Fig.~\ref{fig: model}. First, we extract the features from the input egocentric images $\mathbf{v}$ using a deep network~\cite{he2016deep}. For music $\mathbf{a}$, we adopt the same feature extraction process as EDGE~\cite{tseng2023edge}, leveraging the pre-trained JukeBox~\cite{dhariwal2020jukebox} model to capture high-level music features, which are then processed by a transformer encoder~\cite{vaswani2017attention} to produce the final music embedding. The visual and music embeddings are subsequently aligned and integrated by the Fusion Module to form a joined embedding, denoted as \(\mathbf{z}\). This embedding is then fed into a conditional diffusion denoising process, which outputs the final denoised dance motion sequence. The denoising process is guided by the proposed Skeleton Mamba, which is designed to maintain the skeletal structure of the human pose and enhance the smoothness of the generated motion. We use feature-wise linear DenseFiLM~\cite{perez2018film} for timestep encoding and a Cross Attention layer to integrate the condition $\mathbf{z}$ into the denoising process.

\subsection{Skeleton Mamba}
\textbf{Motivation.} Estimating human pose from egocentric video and music requires spatial, temporal, and visual understanding. Previous methods~\cite{li2023ego, zhang2024motion, tseng2023edge} for human motion generation often overlook the structured patterns inherent in the human body. Skeleton-based methods, widely used in action recognition tasks~\cite{yan2018spatial, plizzari2021spatial}, effectively capture essential motion dynamics while handling challenges like partial visibility and occlusions common in egocentric perspectives. These methods show potential for our task, where a semantic understanding of body part dynamics is crucial. To address these needs, we propose Skeleton Mamba, a model that learns detailed body structures to enhance pose accuracy, ensuring head alignment with the egocentric view and synchronizing body movements with musical cues. Our Skeleton Mamba includes the Group Scan and Joint Scan strategy to learn the group-level representation (e.g., left arm, right arm) and joint-level dependencies within the human body. Then a Temporal Scan is applied to capture sequential dependencies over time. Our Skeleton Mamba is thus capable of effectively modeling both spatial and temporal dynamics within human motion data.


\textbf{Human Tokenizer.} Given the human representation $\mathbf{x}\in \mathbb{R}^{T \times J \times D}$, where $T$ is the number of frames, $J$ is the number of joints, and $D$ is the dimension, we first tokenize the human pose into $G$ overlapping joint groups, each group containing $P$ joints. This results in a group sequence represented as \( \mathbf{g} \in \mathbb{R}^{T \times G \times E} \):
\begin{equation}
   \mathbf{g}=[\mathbf{g}_1, \mathbf{g}_2, \ldots, \mathbf{g}_G] = \text{HumanTokenizer}(\mathbf{x})\Comma
\end{equation}
where $\mathbf{g}_i \in \mathbb{R}^{T\times 1 \times E}$ represents the embedding of each token in the group sequence, and $E=P\times D$ is the embedding dimension. The joints of the human body are grouped based on skeletal parts, with some joints shared across multiple groups. This grouping strategy reflects the symmetrical structure of the human body. 


\textbf{Group Scan with Multi-directional SSD.} 
Inspired by recent works \cite{zhang2024motion, wang2024text} that employed State Space Model (SSM)~\cite{gu2023mamba} for human motion data, we extend these ideas by adopting a multi-directional approach within the grouped structure, leveraging the State Space Duality (SSD)~\cite{mamba2} to encode human motion at the group level.
Our Group Scan rearranges the group tokens in multiple ways and learns them from multiple directions to enable a more comprehensive exploration of token relationships. Algorithm~\ref{alg: MSM} shows our Multi-directional SSD (MSSD). We first utilize \( n \) permutation operators \([\pi_1, \pi_2, \ldots, \pi_n]\), where each \(\pi_i\) is the element of symmetric group \(\text{Sym}(G)\). Each \(\pi_i\) reorders the group tokens of the group embedding \(\mathbf{g} \in \mathbb{R}^{T \times G \times E}\). For each permutation \(\pi_i\), the reordered embedding \(\mathbf{g}_i\) is obtained as \(\mathbf{g}^{\pi_i} = [\mathbf{g}_{\pi_i(1)}, \ldots, \mathbf{g}_{\pi_i(G)}] \in \mathbb{R}^{T \times G \times E}\). These \( n \) reordered embeddings are concatenated along the group dimension, forming a combined representation \(\bar{\mathbf{g}}_c\in \mathbb{R}^{T \times nG \times E}\). Embedding $\bar{\mathbf{g}}_c$ is processed by an SSD, and the resulting output is split back into \( n \) segments: \([\bar{\mathbf{g}}^{\pi_1}, \bar{\mathbf{g}}^{\pi_2}, \ldots, \bar{\mathbf{g}}^{\pi_n}]~\text{with}~\bar{\mathbf{g}}^{\pi_i} \in \mathbb{R}^{T \times G \times E}\). Each segment is then reordered back to its original token order using the corresponding inverse permutation \(\pi_i^{-1}\). Finally, the mean of the reordered embeddings is taken to produce the transformed group embedding \(\mathbf{y} \in \mathbb{R}^{T \times G \times E}\),
where $\mathbf{y}_i \in \mathbb{R}^{T\times 1 \times E}$ represents the $i$-th token in the transformed group sequence. 
\begin{algorithm}
\caption{Multi-directional SSD (MSSD)}
\label{alg: MSM}
\textbf{Input:} Group Embedding $\mathbf{g} : (T, G, E)$; \\
$n$ permutation $(\pi_1,\pi_2,\ldots, \pi_n), \pi_i\in \text{Sym}(G)$. \\
\textbf{Output:} Transformed group embedding $\mathbf{y} : (T, G, E)$.

\begin{algorithmic}[1]
\For{$i =1,\dots, n$}
    \State \textit{/* permutation sequence */}
    \State $\mathbf{g}^{\pi_i}: (T, G, E)=[\mathbf{g}_{\pi_i(1)}, \ldots, \mathbf{g}_{\pi_i(G)}]$
\EndFor
\State $\mathbf{g}_c: (T, nG, E) \gets \text{Concat}([\mathbf{g}^{\pi_1}, \mathbf{g}^{\pi_2}, ..., \mathbf{g}^{\pi_n}])$
\State $\bar{\mathbf{g}}_c: (T, nG, E) \gets \text{SSD}(\mathbf{g}_c)$
\State $[\bar{\mathbf{g}}^{\pi_1}, \bar{\mathbf{g}}^{\pi_2},\ldots, \bar{\mathbf{g}}^{\pi_n}]\gets \text{Split}(\bar{\mathbf{g}}_c)$
\For{$i = 1, \dots, n$}
    \State \textit{/* reverse to the original order */}
    \State $\bar{\mathbf{g}}^{\pi^{-1}_i}: (T, G, E)=[\bar{\mathbf{g}}_{\pi^{-1}_{i}(1)},\ldots, \bar{\mathbf{g}}_{\pi^{-1}_{i}(G)}]$
\EndFor
\State $\mathbf{y}: (T, G, E) \gets \text{Mean}([\bar{\mathbf{g}}^{\pi^{-1}_1}, \bar{\mathbf{g}}^{\pi^{-1}_2}, \ldots, \bar{\mathbf{g}}^{\pi^{-1}_n}])$
\State \textbf{Return} $\mathbf{y}$
\end{algorithmic}
\end{algorithm}

\textbf{Joint Scan.} To learn the human motion at the joint level, we transform each group embedding $\mathbf{y}_i$ to a sequence of individual joints represented as \( \mathbf{y}^\prime_i \in \mathbb{R}^{T \times P \times D}\) using a linear layer and rearrange operator. Each of these sequences is then processed by an SSD block to obtain transformed joint sequence embedding $\mathbf{y}^{\prime\prime}_i \in \mathbb{R}^{T \times P \times D}$:
\begin{equation}
\begin{aligned}
    \mathbf{y}^\prime_i &= \text{Rearrange}(\text{Linear}(\mathbf{y}_i))\Comma \\
    \mathbf{y}^{\prime\prime}_{i} &= \text{SSD}(\mathbf{y}^\prime_{i})\FullStop
\end{aligned}
\end{equation}
In total, $G$ separate SSD modules with shared weights are employed. The objective is to learn detailed intra-group dependencies, ensuring that the interactions within each joint sequence are effectively captured. Unlike group tokens, joint tokens have an inherent order, allowing a unidirectional SSD to be sufficient for learning. By integrating the Group Scan and Joint Scan, our approach effectively captures both high-level and detailed spatial dependencies in the human pose structure. All outputs of the SSD modules are then concatenated and fed to the Inverse Human Tokenizer to restore the original pose shape $\mathbf{t}\in \mathbb{R}^{T\times J\times D}$. 
\begin{equation}
    \mathbf{t} = \text{InverseHumanTokenizer}((\text{Concat}(\mathbf{y}^{\prime\prime}_{1}, \mathbf{y}^{\prime\prime}_{2}, \dots, \mathbf{y}^{\prime\prime}_{G})))\FullStop
\end{equation}

\textbf{Temporal Scan.}  We apply Temporal Scan to model the temporal dependencies across the temporal domain, enhancing the representation of dynamic pose changes over time. First, we swap the input dimension from $\mathbf{t}\in \mathbb{R}^{T\times J \times D}$ to $\mathbf{t}^\prime \in \mathbb{R}^{J \times T \times D}$. Then, the embedding $\mathbf{t}^\prime$ is processed by two SSD modules that scan the sequence in both backward and forward directions as in~\cite{wang2024text, li2025videomamba}, producing $\mathbf{t}_{backward}$ and $\mathbf{t}_{forward}$. Unlike Group Scan, which uses multiple directions, temporal sequences must preserve natural dependencies over time, allowing only backward and forward scanning. 

\begin{equation}\label{eq: temporal}
\begin{aligned}
    \mathbf{t}_{backward} &= \text{BackwardSSD}(\mathbf{t}^\prime)\Comma \\
    \mathbf{t}_{forward} &= \text{ForwardSSD}(\mathbf{t}^\prime) \FullStop
\end{aligned}
\end{equation}


\begin{theorem}
\label{prop: prop_1}
    Let \( S_J = \text{Sym}(J) \) denote the symmetric group of \( J \) elements, and let \( \text{HT}, \text{HT}^{-1} \) be short for the HumanTokenizer and InverseHumanTokenizer respectively. Suppose that the $HT(\cdot)$ and $HT^{-1}(\cdot)$ operator are fixed, such that the set $H=\{\sigma \in S_J | HT(\mathbf{x}^\sigma) = HT(\mathbf{x}) \}$ is a non-empty subgroup of $S_J$. Then for arbitrary continuous and $H$-equivariant function $g: \mathbb{R}^{J\times D} \rightarrow \mathbb{R}^{J\times D}$, compact set $K \subseteq \mathbb{R}^{J\times D}$, and $\epsilon>0$, there exists a function  $f: \mathbb{R}^{J\times D} \rightarrow \mathbb{R}^{J\times D}$ that constructed by our Skeleton Mamba such that:
\begin{equation*}
    ||f(\mathbf{x})-g(\mathbf{x})||_\infty < \epsilon,~\forall \mathbf{x} \in K.
\end{equation*}
\end{theorem}
\begin{proof}
    See Supplementary Material.
\end{proof}
\begin{remark}
The group \( H \) in the above theorem represents the set of human body symmetries. The function \( g \), which is \( H \)-equivariant, represents the unknown target function we aim to learn. In practice, \( g \) must respect the symmetries of the human body, meaning it must be \( H \)-equivariant. Thus, the assumption that \( g \) is \( H \)-equivariant is natural. Intuitively, Theorem~\ref{prop: prop_1} asserts that our Skeleton Mamba can effectively learn complex human motions, including those requiring precise coordination to follow both egocentric views and musical cues, while maintaining the equivariant properties associated with human body symmetries.
\end{remark}
\subsection{Training and Inference}
\textbf{Auxiliary Loss.} As in state-of-the-art work in human motion generation~\cite{tseng2023edge, siyao2022bailando, beatit}, we employ position loss \(\mathcal{L}_{pos}\) for accurate joint positioning and velocity loss \(\mathcal{L}_{vel}\) for smooth motion dynamics. To reduce the foot sliding effects, we also use the contact loss \(\mathcal{L}_{contact}\) as in~\cite{tseng2023edge}. The kinematic loss is expressed as follows:
\begin{equation}
    \mathcal{L}_{kin} = \lambda_{pos} \mathcal{L}_{pos} + \lambda_{vel} \mathcal{L}_{vel} + \lambda_{contact} \mathcal{L}_{contact}\FullStop
\end{equation}

\textbf{Ego-Music Alignment Loss.} We use \(\mathcal{L}_{align}\) to align egocentric video and music at the temporal level. This loss ensures the high-dimensional features from both modalities are unified in a shared space, enabling the denoising model to generate human motion that is more coherent and contextually aligned with both inputs. Given the music embedding \( \mathbf{z}_a \in \mathbb{R}^{T \times D_c} \) and egocentric vision embedding \( \mathbf{z}_v \in \mathbb{R}^{T \times D_c} \). This loss averages the symmetrical contrastive loss between vision and music embeddings: 
\begin{equation}\label{eq: align}
    \mathcal{L}_{align} = - \frac{1}{T} \sum_{i=1}^{T} \log \frac{\exp(\text{sim}(\mathbf{z}_a^i, \mathbf{z}_v^i) / \tau)}{\sum_{j=1}^{T} \exp(\text{sim}(\mathbf{z}_a^i, \mathbf{z}^j_v) / \tau)}\Comma
\end{equation}
where \(\text{sim}(\mathbf{z}_a^i, \mathbf{z}^j_v)\) represents the cosine similarity between the music embedding at frame \(i\) and the vision embedding at frame \(j\). The parameter \(\tau\) is a temperature scalar that controls the sharpness of the similarity distribution.

We employ the diffusion loss $\mathcal{L}_{simple}$ as in~\cite{ho2020denoising}. The total training loss can be formulated as follows:
\begin{equation}
    \mathcal{L}_{total} = \mathcal{L}_{simple} + \lambda_{kin} \mathcal{L}_{kin} + \lambda_{align} \mathcal{L}_{align}\FullStop
\end{equation}

\textbf{Head Guidance Sampling.} To enhance the consistency of the head movements in the generated dance with egocentric images, we define a goal function \(\mathcal{G}_{head}(\cdot)\) that guides the head to align closely with the head pose estimated from the egocentric images. This goal function consists of two components: a positional alignment term $\mathcal{G}_{pos}(\cdot)$ and a rotational alignment term $\mathcal{G}_{rot}(\cdot)$.
\begin{equation}\label{eq: headalign}
\begin{aligned}
    \mathcal{G}_{head}(\mathbf{x}) &= \gamma_{pos}\mathcal{\mathcal{G}}_{pos}(\mathbf{x}) + \gamma_{rot}\mathcal{G}_{rot}(\mathbf{x})\Comma \\
    \mathcal{G}_{pos}(\mathbf{x}) &= \frac{1}{T} \sum_{i=1}^{T} \|\mathbf{p}^{i} - \hat{\mathbf{p}}^{i}\|^2\Comma  \\
    \mathcal{G}_{rot}(\mathbf{x}) &= \frac{1}{T} \sum_{i=1}^{T} \|\log\left(\mathbf{R}^i \hat{\mathbf{R}}^{i \top}\right)\|_F^2\Comma
\end{aligned}
\end{equation}
where \( \mathbf{p}^i \) and \( \mathbf{R}^i \) represent the global head position and global head rotation matrix of the generated motion, respectively, and \( \hat{\mathbf{p}}^i \) and \( \hat{\mathbf{R}}^i \) are the corresponding estimated values calculated using the hybrid approach proposed in~\cite{li2023ego}.

With goal function $\mathcal{G}_{head}(\cdot)$, we formulate the guided sampling problem as optimizing the probability of constraint satisfaction:
\begin{equation}
\begin{aligned}
    p(\mathbf{x}_0 | \mathcal{O}=1, \mathbf{z})& \propto p_\theta(\mathbf{x}_0|\mathbf{z})p(\mathcal{O}=1|\mathbf{x}_0, \mathbf{z}) \\
    & \propto p_\theta(\mathbf{x}_0|\mathbf{z})\cdot \text{exp}(\mathcal{G}_{head}(\cdot))\Comma
\end{aligned} 
\end{equation}
where $\mathcal{O}$ is an indicator to check if the generated dance motion $\mathbf{x}_m$ at denoising step $m$ reaches the goal $\mathcal{G}_{head}(\cdot)$. Similar to~\cite{huang2023diffusion}, we use the first order Taylor expansion around $\mathbf{x}_m=\mathbf{\mu}$ to estimate $p(\mathcal{O}=1| \mathbf{x}_m, \mathbf{z})$:
\begin{equation}
    \text{log } p(\mathcal{O}=1|\mathbf{x}_m, \mathbf{z}) \approx (\mathbf{x}_m - \mu)\xi + \mathcal{C}\Comma
\end{equation}
where $\mu = \mu_\theta(\mathbf{x}_m, m, \mathbf{z})$, $\mathcal{C}$ is a constant, and $\xi$ is calculated as follows:
\begin{equation}
\begin{aligned}
    \xi & = \nabla_{\mathbf{x}_m} \text{log } p(\mathcal{O}=1|\mathbf{x}_m, \mathbf{z}) \big|_{\mathbf{x}_m = \mu} \\
    & = \nabla_{\mathbf{x}_m} \mathcal{G}_{head}(\cdot) \big|_{\mathbf{x}_m = \mu}\FullStop
\end{aligned}
\end{equation}
Hence, we have the sampling process with a goal function:
\begin{equation}
    p_\theta(\mathbf{x}_{m-1}|\mathbf{x}_m, \mathcal{O}=1, \mathbf{z}) = \mathcal{N}(\mathbf{x}_{m-1};\mu + \lambda\Sigma\xi,\Sigma)\Comma
\end{equation}
here $\Sigma=\Sigma_\theta(\mathbf{x}_m,m,\mathbf{z})$ and $\lambda$ is the scaling factor.

\graphicspath{./images} 
\section{Experiments}
\textbf{Baselines.} We compare our method EgoMusic Motion Network (EMM) with egocentric image-driven motion estimation works (PoseReg~\cite{yuan2019ego}, Kinpoly~\cite{luo2021dynamics}, EgoEgo~\cite{li2023ego}) and music-driven motion generation works (FACT~\cite{li2021ai}, Bailando~\cite{siyao2022bailando}, EDGE~\cite{tseng2023edge}). Since our task involves both egocentric video and music, for a fair comparison, we incorporate a music encoder, Jukebox~\cite{dhariwal2020jukebox}, to process audio input for PoseReg~\cite{yuan2019ego}, Kinpoly~\cite{luo2021dynamics}, EgoEgo~\cite{li2023ego}. For the music-driven works, we add visual features from the egocentric video using~\cite{he2016deep}. The implementation details of all baselines can be found in our Supplementary Material.

\textbf{Metrics.} Following~\cite{li2023ego}, we evaluate our method using five standard metrics commonly used in human pose estimation task: \textit{i)} $\mathbf{O_{\text{head}}}$, which measures the Frobenius norm between the predicted and actual head rotation matrices. \textit{ii)} $\mathbf{T_{\text{head}}}$, calculated as the average Euclidean distance between the predicted and true head translation. \textit{iii)} MPJPE, the Mean Per Joint Position Error, quantifies the average discrepancy in joint positions. \textit{iv)} Accel refers to the difference in acceleration between predicted and ground truth joint positions. \textit{v)} FS assesses foot skating, capturing unnatural foot movement. \textit{vi)} To evaluate how well the generated dance motions align with the music and egocentric video, we propose a new metric called the Motion-Music-Vision (MMV) (please refer to our Supplementary Material). 
\subsection{Main Results}
\begin{table}[ht]
\vspace{-1.5ex}
\centering
\renewcommand
\tabcolsep{2pt}
\hspace{1ex}
\vskip 0.1 in
\resizebox{\linewidth}{!}{
\begin{tabular}{@{}rcccccc@{}}
\toprule
Baseline & $\textbf{O}_{head}$$\downarrow$ & $\textbf{T}_{head}$$\downarrow$ & MPJPE$\downarrow$ & Accel$\downarrow$ & FS$\downarrow$ & MMV$\uparrow$\\
\midrule
Pose-Reg~\cite{ng2020you2me} & 1.78 & 423.56 & 351.37 & 37.14 & 98.79 & 0.182 \\
Kinpoly~\cite{luo2021dynamics} & 1.16 & 392.67 & 338.74 & 16.27 & 25.81 & 0.197 \\
EgoEgo~\cite{li2023ego} & 0.74 & 373.67 & 152.02 & 14.23 & 22.13 & 0.218 \\ \midrule
FACT~\cite{li2021ai} & 1.54 & 407.89 & 173.68 & 14.61 & 15.07 & 0.202 \\
Bailando~\cite{siyao2022bailando} & 1.57 & 411.43 & 175.31 & 14.72 & 15.46 & 0.210 \\
EDGE~\cite{tseng2023edge} & 1.52 & 404.62 & 167.43 & 14.37 & 14.75 & 0.224 \\ \midrule
EMM (music only) & 1.43 & 398.41 & 157.36 & 14.28 & 14.04 & - \\
EMM (ego only) & 0.61 & 355.16 & 186.49 & 16.02 & 13.45 & - \\
EMM (ego + music) & \textbf{0.53} & \textbf{342.37} & \textbf{137.54} & \textbf{11.84} & \textbf{12.79} & \textbf{0.262} \\ \bottomrule
\end{tabular}}
\vspace{0ex}
\caption{\textbf{Human motion estimation results.}}
\label{table: main-results}
\vspace{-2ex}
\end{table}

\begin{figure*}[ht]
    \centering
    \subfloat[Transformer]{\includegraphics[width=0.25\linewidth]{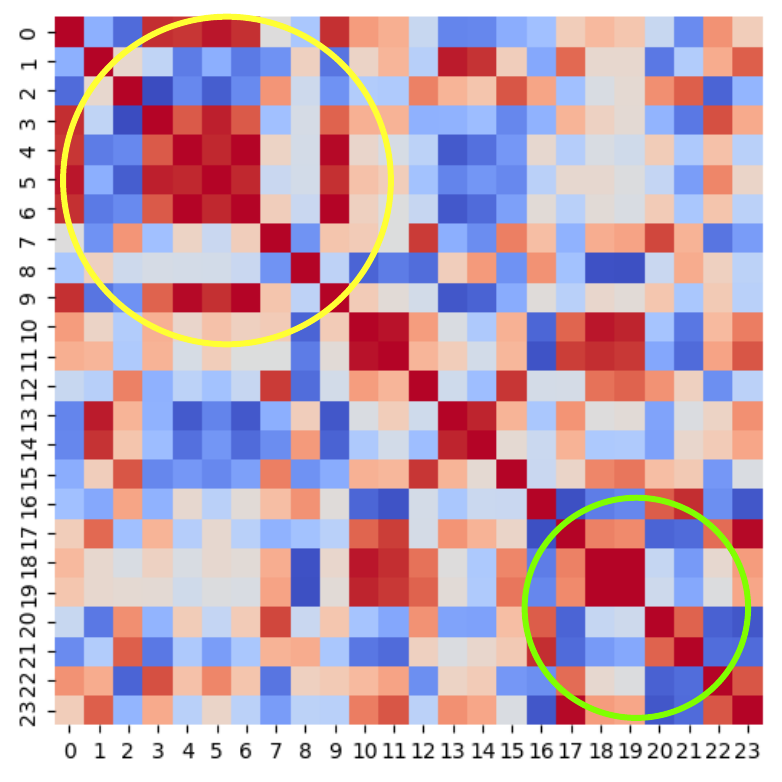}} \hfill
    \subfloat[Vanilla Mamba]{\includegraphics[width=0.25\linewidth]{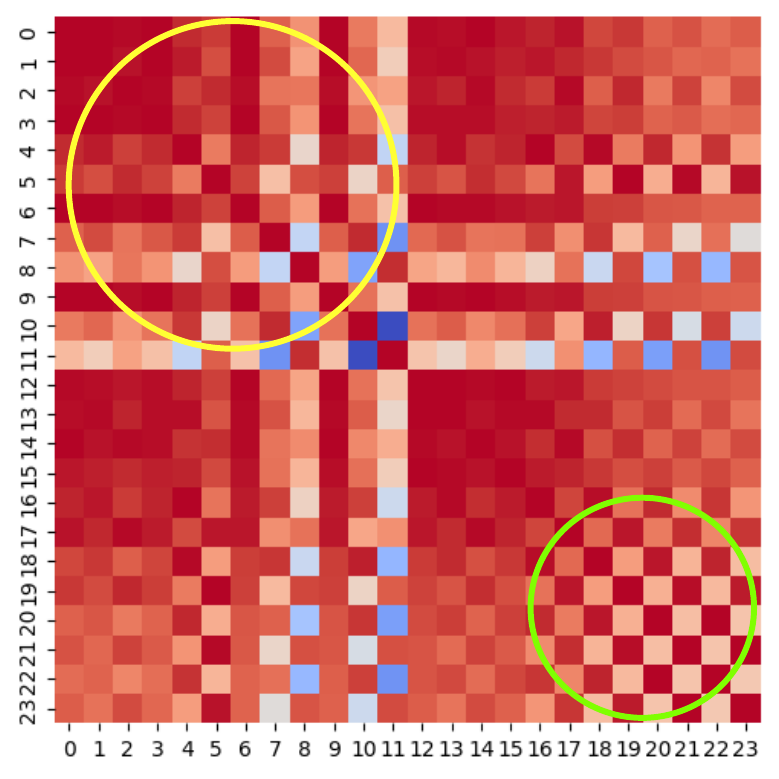}} \hfill
    \subfloat[Skeleton Mamba (ours)]{\includegraphics[width=0.303\linewidth]{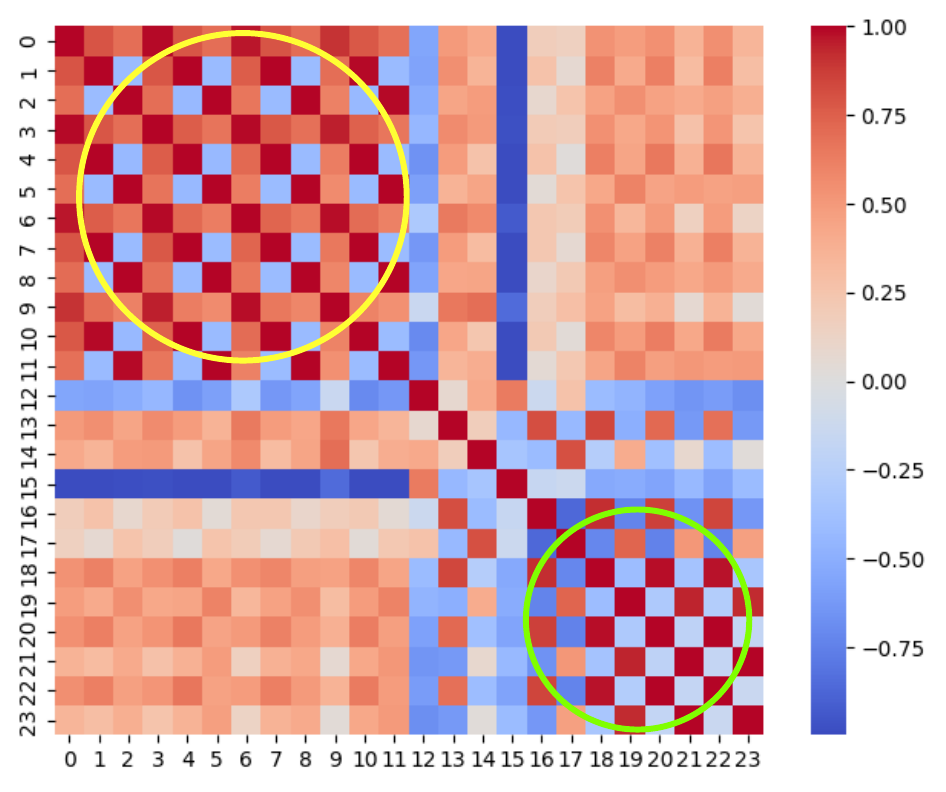}} \hfill
    \subfloat[Human Skeleton]{\includegraphics[width=0.19\linewidth]{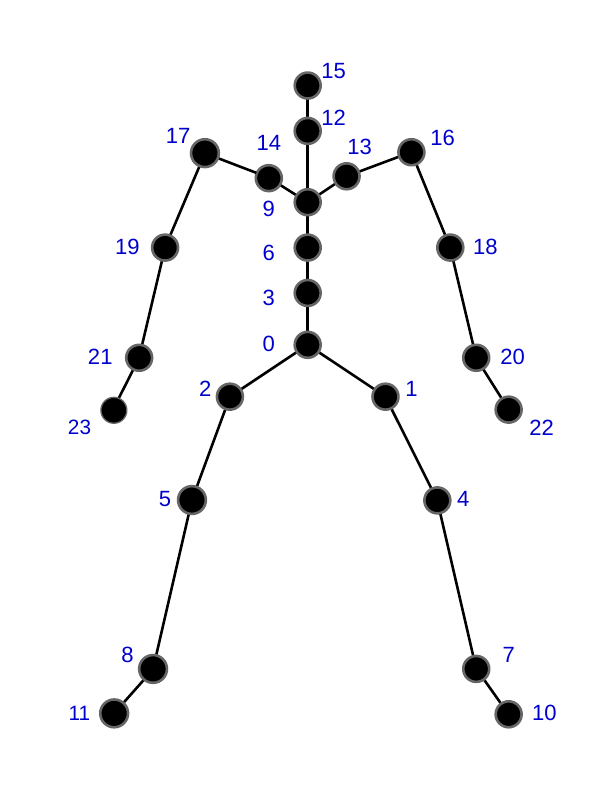}\label{fig:cosin_d}}
    \vspace{-1ex}
    \caption{\textbf{Cosine similarity of joint embeddings.} We calculate the cosine similarity of joint embeddings of: a) Transformer~\cite{tseng2023edge}, b) Vanilla Mamba~\cite{mamba2}, and c) Skelton Mamba (ours).  d) Visualization of the human body with 24 joints for reference. We highlight similarities in arm joints (green circle) and leg joints (yellow circle). Our method shows a clear distinction between left and right limbs.}
    \label{fig: cosine}
    \vspace{-2ex}
\end{figure*}

Table~\ref{table: main-results} presents the performance comparison between our method and other baselines. Table~\ref{table: main-results} shows that our EMM outperforms other baselines. EMM reduces head orientation error $\textbf{O}_{head}$ by \SI{0.21}{\radian} and head translation error $\textbf{T}_{head}$ by \SI{31.3}{\mm} compared to EgoEgo~\cite{li2023ego}. Furthermore, when compared to methods conditioned solely on music~\cite{li2021ai,siyao2022bailando,tseng2023edge}, our approach offers notable improvements in both accuracy in MPJPE and physical plausibility in FS metric. We also investigate the impact of each input modality, with results confirming that combining egocentric and music inputs enhances overall motion accuracy.

\subsection{Cross-dataset Results} To validate the generalization of our method, we conduct a cross-dataset experiment. We use the pretrained weights of all methods, originally trained on the EgoAIST++ dataset, to test on the EgoExo4D dataset~\cite{grauman2024ego}. EgoExo4D contains approximately two hours of dance sequences, along with egocentric videos and music captured in real-world. 
Table~\ref{table: cross-dataset} shows that our method consistently outperforms other baselines in the cross-dataset experiment. 
\begin{table}[t]
\vspace{-3ex}
\centering
\renewcommand
\tabcolsep{2pt}
\hspace{1ex}
\vskip 0.1 in
\resizebox{\linewidth}{!}{
\begin{tabular}{@{}rcccccc@{}}
\toprule
Baseline & $\textbf{O}_{head}$$\downarrow$ & $\textbf{T}_{head}$$\downarrow$ & MPJPE$\downarrow$ & Accel$\downarrow$ & FS$\downarrow$ & MMV$\uparrow$\\
\midrule
Pose-Reg~\cite{ng2020you2me} & 1.21 & 642.47 & 377.56 & 30.36 & 56.18 & 0.165 \\
Kinpoly~\cite{luo2021dynamics} & 0.78 & 354.19 & 251.27 & 17.84 & 25.31 & 0.187 \\
EgoEgo~\cite{li2023ego} & 0.67 & 347.23 & 234.58 & 16.76 & 20.15 & 0.203 \\ \midrule
FACT~\cite{li2021ai} & 1.37 & 685.81 & 244.89 & 17.54 & 18.53 & 0.195 \\
Bailando~\cite{siyao2022bailando} & 1.44 & 688.54 & 231.77 & 14.23 & 18.67 & 0.211 \\
EDGE~\cite{tseng2023edge} & 1.27 & 644.62 & 213.37 & 13.78 & 14.33 & 0.221 \\ \midrule
EMM (Ours) & \textbf{0.61} & \textbf{322.19} & \textbf{191.55} & \textbf{12.76} & \textbf{13.18} & \textbf{0.239} \\ \bottomrule
\end{tabular}}
\vspace{-2ex}
\caption{\textbf{Cross-dataset experiment results.}}
\label{table: cross-dataset}
\vspace{-3ex}
\end{table}

\subsection{Skeleton Mamba Analysis}
\textbf{Can Skeleton Mamba learn human skeleton?} To answer this, we compute the cosine similarity matrices of joint embeddings for Transformer~\cite{tseng2023edge}, Vanilla Mamba~\cite{mamba2}, and our Skeleton Mamba. Fig.~\ref{fig: cosine} shows that Skeleton Mamba captures the human body structure more clearly. We highlight the similarity between joints in the arms (green circles) and the legs (yellow circles) in Fig.~\ref{fig: cosine}. In particular, the joints of the right arm (join 17, 19, 21, and 23 in Fig.~\ref{fig:cosin_d}) exhibit strong similarity to each other. By contrast, joint 20, which belongs to the left arm, shows low similarity with the joints in the right arm, indicating effective separation between limbs. In comparison, Transformer~\cite{tseng2023edge} and Vanilla Mamba~\cite{mamba2} models display less distinct differentiation.

\begin{table}[ht]
\centering
\setlength{\tabcolsep}{0.14 em}
\renewcommand{\arraystretch}{1.3}
\resizebox{\linewidth}{!}{
\begin{tabular}{c|lcccccc}
\toprule
Dimension & Scan Type & $\textbf{O}_{head}$$\downarrow$ & $\textbf{T}_{head}$$\downarrow$ & MPJPE$\downarrow$ & Accel$\downarrow$ & FS$\downarrow$ & MMV$\uparrow$ \\ \midrule
\multirow{2}{*}{Spatial} & Unidirectional~\cite{mamba2} & 0.63 & 386.12 & 192.38 & 16.21 & 14.65 & 0.245 \\
 & Bidirectional~\cite{zhang2024motion} & 0.58 & 357.81 & 160.12 & 14.02 & 14.12 & 0.255 \\ \midrule
Temporal & Unidirectional~\cite{mamba2} & 0.61 & 371.62 & 181.43 & 15.29 & 15.47 & 0.251 \\ \midrule
Temporal \& Spatial & Skeleton Mamba & \textbf{0.53} & \textbf{342.37} & \textbf{137.54} & \textbf{11.84} & \textbf{12.79} & \textbf{0.262} \\
\bottomrule
\end{tabular}}
\vspace{-2ex}
\caption{\textbf{Scanning strategy analysis}. 
}
\label{tab: scan_ablation}
\vspace{-2ex}
\end{table}
\textbf{Scan strategy analysis.} Table~\ref{tab: scan_ablation} shows the impact of different scanning schemes on learning human motion. When applying the scanning to the spatial domain, unidirectional scanning~\cite{mamba2} and bidirectional scanning~\cite{zhang2024motion} show reasonable results but are lower than our Skeleton Mamba. 
Table~\ref{tab: scan_ablation} also shows that our method, which uses bidirectional scanning (i.e., ForwardSSD and BackwardSDD in Equation~\ref{eq: temporal}) for temporal processing, outperforms the unidirectional approach. 
This experiment and Theorem~\ref{prop: prop_1} confirm that by learning in both temporal and spatial domains, our Skeleton Mamba can effectively handle the geometry of human motion that aligns with both egocentric and music input.


\begin{figure}[t]
\includegraphics[width=\linewidth]{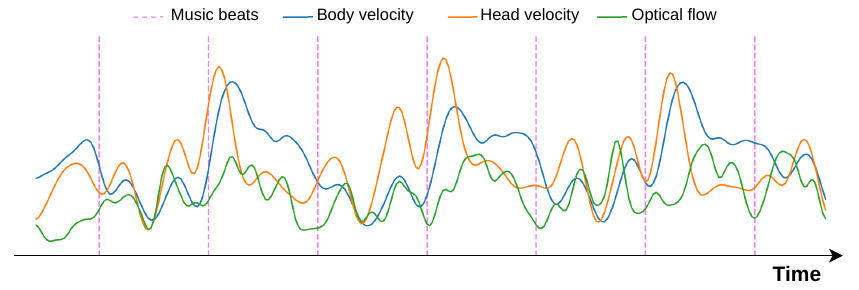}
    \vspace{-5ex}\caption{\label{fig: correlation} 
    \textbf{Motion, music, and egocentric view correlation.}
    }
    \vspace{-4ex}
\end{figure}

\subsection{Ablation Study}
\textbf{Motion, music, and egocentric correlation.} We plot the kinematic velocity, music beats, and optical flow extracted from the egocentric video to visualize the correlation between these three. The music beats are calculated using the beat extracting algorithm~\cite{mcfee2015librosa}. The motion beats are extracted as the local extrema of the kinematic velocity. The optical flow is extracted from the egocentric video using RAFT~\cite{teed2020raft}. Fig.~\ref{fig: correlation} shows that the generated dance aligns well with the music beat and optical flow. 
The results show that our model effectively synchronizes both body and head movements with the respective audio and visual cues, generating well-coordinated motion from multimodal inputs.

\begin{table}[t]
\centering
\vspace{-2ex}
\setlength{\tabcolsep}{0.2 em}
\renewcommand{\arraystretch}{1}
\resizebox{\linewidth}{!}{
\begin{tabular}{lcccccc}
\toprule
Ablation & $\textbf{O}_{head}$$\downarrow$ & $\textbf{T}_{head}$$\downarrow$ & MPJPE$ \downarrow$ & Accel$\downarrow$ & FS$\downarrow$ & MMV$\uparrow$ \\ \midrule
EMM & 0.53 & 342.37 & 137.54 & 11.84 & 12.79 & 0.262 \\
EMM w.o. $\mathcal{G}_{head}$ & 0.76 & 374.86 & 146.38 & 12.71 & 13.34 & 0.254 \\
EMM w.o. $\mathcal{L}_{align}$ & 0.56 & 350.06 & 142.47 & 12.16 & 13.09 & 0.242 \\ \bottomrule
\end{tabular}}
\vspace{-0.5ex}
\caption{\textbf{Loss analysis}.}
\label{tab: loss_anal}
\end{table}

\textbf{Loss analysis.} Table~\ref{tab: loss_anal} shows that each loss function plays an essential role in our network. In particular, $\mathcal{L}_{align}$ positively impacts all metrics, especially in terms of motion-music-vision synchronization. Additionally, the head guidance loss $\mathcal{G}_{head}$,  applied during the sampling process, substantially increases the accuracy of head movements, ensuring better consistency with the egocentric view.

\textbf{Quality results.} Fig.~\ref{fig: compare_fig} shows qualitative comparisons between EDGE~\cite{tseng2023edge}, EgoEgo~\cite{li2023ego} and our method. EgoEgo~\cite{li2023ego} has difficulty generating dance motions that follow the choreography. EDGE~\cite{tseng2023edge} generates dance movements in sync with the music, but its head movements do not correspond to the egocentric input. In contrast, our approach demonstrates a more cohesive generation of both dance and head movements. 

\begin{figure}[t] 
   \centering
   \normalsize
\resizebox{\linewidth}{!}{
\setlength{\tabcolsep}{2pt}
\begin{tabular}{ccc}
\rotatebox[origin=l]{90}{\hspace{0.23cm} 
{\begin{tabular}[c]{@{}c@{}}View\end{tabular}}} &
\shortstack{\includegraphics[width=\linewidth]{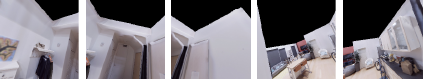}} \\[0.01pt]

\rotatebox[origin=l]{90}{\hspace{0.1cm}
{\begin{tabular}[c]{@{}c@{}}EDGE~\cite{tseng2023edge}\end{tabular}}} &
\shortstack{\includegraphics[width=\linewidth]{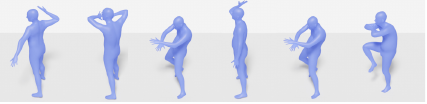}} \\[0.01pt]

\rotatebox[origin=l]{90}{\hspace{0.01cm} 
{\begin{tabular}[c]{@{}c@{}}EgoEgo~\cite{li2023ego}\end{tabular}}} &
\shortstack{\includegraphics[width=\linewidth]{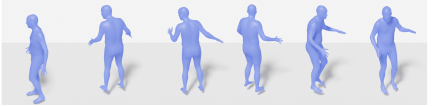}} \\[0.01pt]

\rotatebox[origin=l]{90}{\hspace{0.6cm} 
{\begin{tabular}[c]{@{}c@{}}Ours\end{tabular}}} &
\shortstack{\includegraphics[width=\linewidth]{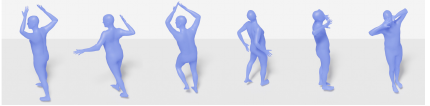}} \\[0.01pt]

\rotatebox[origin=l]{90}{\hspace{-0.1cm} 
{\begin{tabular}[c]{@{}c@{}}Ground Truth\end{tabular}}} &
\shortstack{\includegraphics[width=\linewidth]{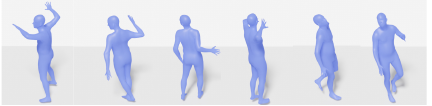}} \\[0.01pt]
\end{tabular}}
    \vspace{-2ex}\caption{\textbf{Qualitative comparison.} Our approach produces well-aligned human motion with ego view and music.}
    \label{fig: compare_fig}
    \vspace{-3ex}
\end{figure}


\subsection{Skeleton Mamba on Human Motion Tasks}
To further evaluate the effectiveness of our Skeleton Mamba in modeling human motion, we conduct experiments on text-to-human motion generation and skeleton-based action recognition tasks. Implementation details of both tasks can be found in our Supplementary Material. 

\textbf{Text-to-motion generation.} We evaluate our Skeleton Mamba on HumanML3D~\cite{guo2022generating} dataset. We compare with recent work MDM~\cite{tevet2023human}, MLD~\cite{chen2023executing} and \cite{zhang2024motion} using metrics in~\cite{guo2022generating}. Table~\ref{tab: text2motion_re} shows that our method outperforms other baselines in the text-to-human motion generation task.

\begin{table}[t]
\vspace{-1.2ex}
\renewcommand\arraystretch{1.2}
\renewcommand{\tabcolsep}{1.0pt}
\small
\centering
\parbox[b]{0.480 \linewidth}{
\resizebox{0.97\linewidth}{!}{
\begin{tabular}{lccccc}
\toprule
Method   & R Prec.$\uparrow$& MM Dist$\downarrow$ & Div$\rightarrow$  \\ \midrule
Ground Truth      & 0.797              & 2.974        & 9.503  \\ \hline
MDM~\cite{tevet2023human}    & 0.611        & 5.566       & 9.559     \\
MLD~\cite{chen2023executing}  & 0.772    & 3.196   & 9.724       \\
MMamba~\cite{zhang2024motion}    & 0.790  & 3.060  & 9.871 \\
Ours & \textbf{0.795}  & \textbf{2.983}  & \textbf{9.484} \\ \bottomrule
\end{tabular}}
\vspace{-2ex}
\caption{\textbf{Text2motion results.}}
\label{tab: text2motion_re}
}
\hfill
\parbox[b]{0.495\linewidth}{
\resizebox{1.05\linewidth}{!}{
\begin{tabular}{l|cc}\hline
Method & NTU60-XS & Kinetics\\ \hline
MS-G3D~\cite{liu2020disentangling} & 91.5 & 38.0 \\
PoseConv3D~\cite{duan2022revisiting}  & 93.1 & 47.7 \\ 
MotionBERT~\cite{zhu2023motionbert} & 93.0 & - \\ 
DSTA-Net~\cite{shi2020decoupled} & 91.5 & - \\
Ours & \textbf{94.4} & \textbf{52.4} \\ \hline
\end{tabular}
}
\vspace{-2.5ex}
    \caption{\textbf{Action rec. results.}}
    \label{tab: act_reg_re}
}
\vspace{-3ex}
\end{table}

\textbf{Human action recognition.} We benchmark on two datasets: Kinetics400~\cite{kay2017kinetics} and NTU RGB+D 60~\cite{shahroudy2016ntu,liu2019ntu}. We compare with recent methods, including MS-G3D~\cite{liu2020disentangling}, PoseConv3D~\cite{duan2022revisiting}, MotionBERT~\cite{zhu2023motionbert} and DSTA-Net~\cite{shi2020decoupled}. Top-1 classification accuracy is used as the metric. Table~\ref{tab: act_reg_re} shows that our method clearly outperforms the other action recognition baselines.  Tables~\ref{tab: text2motion_re} and \ref{tab: act_reg_re} demonstrate that while our Skeleton Mamban is designed for human motion modeling, it has the potential to generalize effectively to various setups, including generative and recognition tasks.

\vspace{-1ex}
\section{Discussion}
\vspace{-1ex}
\textbf{Broader Impact.} We believe our work represents a significant step toward understanding human motion dynamics and potentially has a profound impact on different fields such as VR/AR, metaverse, or film animation. For instance, in VR dance games~\cite{sarupuri2024dancing, chan2010virtual,laattala2024wave}, current systems rely on egocentric cameras and additional motion tracking devices such as hand VR motion controllers. By fusing music cues with egocentric data, however, full-body motion can be accurately estimated, eliminating the need for additional sensors. Moreover, our Skeleton Mamba has the potential in other tasks such as human motion synthesis~\cite{tevet2023human}, human action recognition~\cite{xu2022vitpose,tripathi20233d,yan2018spatial}, gesture analysis~\cite{liu2021imigue}, and human-object interaction~\cite{kim2021hotr}.

\textbf{Limitations.} Although our method achieves encouraging results, it still presents certain limitations. First, while our model effectively generates smooth motion sequences, it is not fully optimized for producing very long motion sequences due to the reliance on a simple bidirectional scan in the temporal dimension. Second, when the egocentric video and the music input are not appropriately paired, our model may fail to generate coherent and synchronized motion.


\textbf{Conclusion.} We introduce a new task and method to estimate human dance motion from egocentric and music.  
Our network with the core Skeleton Mamba effectively estimates motion that aligns with both visual and musical cues. We further contribute EgoAIST++ dataset, which provides egocentric, music, and dance groundtruth. Intensive experiments show that our method significantly outperforms existing state-of-the-art approaches, and our Skeleton Mamba has the potential in human motion understanding tasks.

{\small
\bibliographystyle{ieee_fullname}
\bibliography{egbib}
}


\end{document}